\documentclass[runningheads]{llncs}

\newcommand {\ent} {\mathrel{{\scriptstyle\mid\!\sim}}}

\newcommand {\sx} {\langle}
\newcommand {\dx} {\rangle}

\newcommand {\enne} {\mathcal{N}}

\newcommand {\tc} {\mid}
\newcommand {\vuoto} {\emptyset}

\newcommand{\tip}{{\bf T}}

\newcommand{\alc}{\mathcal{ALC}}

\newcommand{\alct}{\mathcal{ALC}+\tip}

\newcommand{\sroiqpt}{\mathcal{SROIQ}^{\Pe}\tip}
\newcommand{\alcFt}{\mathcal{ALC}^{\Fe}\tip}

\newcommand{\be}{\begin{enumerate}}
\newcommand{\ee}{\end{enumerate}}

\newcommand{\hide}[1]{}

\def \cases{\left \{\begin{array}{l}}
\def \endcases{\end{array}\right .}

\newcommand {\Fe} {{\bf F}}

\newcommand {\Pe} {{\bf P}}

\newcommand {\ri} {\rightarrow}
\newcommand {\Ri} {\Rightarrow}

\newcommand {\bes} {\begin{description}}
\newcommand{\ens} {\end{description}}

\newcommand {\beq} {\begin{quote}}
\newcommand {\enq} {\end{quote}}
\newcommand {\bit} {\begin{itemize}}
\newcommand {\enit} {\end{itemize}}

\newenvironment{pozz}{\color{black}}{\color{black}}
\usepackage{times}
\usepackage{helvet}
\usepackage{courier}
\usepackage{times}
\usepackage{undertilde}
\usepackage{graphicx}
\usepackage{latexsym}
\usepackage{graphicx}
\usepackage{color}
\usepackage{amssymb}
\usepackage{colortbl}
\usepackage{amsmath}
\usepackage{microtype}
\usepackage{amsmath}
\usepackage{tikz}

\def \ri{\rightarrow}
\def \Ri{\Rightarrow}

\begin{document}
\bibliographystyle{splncs04}

\title{On the KLM properties of a fuzzy DL with Typicality}



\author{Laura Giordano} 

\institute{DISIT - Universit\`a del Piemonte Orientale, 
 Alessandria, Italy  \\
  \email{laura.giordano@uniupo.it} 
}

\authorrunning{ }
\titlerunning{ }

 \maketitle
 

\begin{abstract}

The paper investigates the properties of a fuzzy logic of typicality. The extension of fuzzy logic with a typicality operator was proposed
in recent work to define a fuzzy multipreference semantics for Multilayer Perceptrons, 
by regarding the deep neural network as a conditional knowledge base.
In this paper, we study its properties. First, a monotonic extension of a fuzzy $\alc$ with typicality is considered (called $\alcFt$) and a reformulation the KLM properties of a preferential consequence relation for this logic is devised. Most of the properties are satisfied, depending on the reformulation and on the fuzzy combination functions considered.
We then strengthen $\alcFt$ with a closure construction by introducing a notion of {\em faithful} model of a weighted knowledge base, which generalizes the notion of {\em coherent} model 
of a conditional knowledge base previously introduced, and we study its properties.


\end{abstract}

\section{Introduction}

Preferential approaches have been used to provide axiomatic foundations of non-monoto- nic  and 
common sense reasoning 
\cite{Delgrande:87,Makinson88,Pearl:88,KrausLehmannMagidor:90,Pearl90,whatdoes,BenferhatIJCAI93,Kern-Isberner01}.
They have been extended to description logics (DLs), as well as to the first order case \cite{BeierleJAR17}, to deal with inheritance with exceptions in ontologies,
by allowing for non-strict forms of inclusions,
called {\em typicality or defeasible inclusions}, 
with different preferential semantics \cite{lpar2007,sudafricaniKR,FI09}
and closure constructions \cite{casinistraccia2010,CasiniDL2013,AIJ15,BonattiSauro17,Pensel18,CasiniStracciaM19,AIJ21}. 

In previous work \cite{JELIA2021}, a concept-wise multipreference semantics 
for weighted knowledge bases has been proposed to account for preferences with respect to different concepts, 
by allowing, for a concept $C$, a set of typicality inclusions of the form $\tip(C) \sqsubseteq D$ (meaning ``the typical $C$'s are $D$'s" or ``normally $C$'s are $D$'s") with positive or negative weights. 
The concept-wise multipreference semantics  
has been first introduced  as a semantics for ranked DL knowledge bases  \cite{iclp2020} and extended in   \cite{JELIA2021} to  weighted knowledge bases in the two-valued and fuzzy case, based on a different semantic closure construction, 
still in the spirit of Lehmann's lexicographic closure \cite{whatdoes} and Kern-Isberner's c-representations  \cite{Kern-Isberner01,Kern-Isberner2014}, but exploiting multiple preferences associated to concepts.
A related semantics with multiple preferences has been proposed in the first-order logic setting by Delgrande and Rantsaudis \cite{Delgrande2020}, and an extension of DLs with defeasible role quantifiers and defeasible role inclusions has been developed by Britz and Varzinczak \cite{Britz2018,Britz2019}, by associating multiple  preferences to roles.

The concept-wise multipreference semantics 
has been proved to have some desired properties from the knowledge representation point of view in the two-valued case \cite{iclp2020}. In particular, it satisfies the KLM postulates of a preferential consequence relation. 
The properties 
of entailment in a fuzzy DL with typicality have not been studied so far. In this paper, 
a monotonic extension of a fuzzy $\alc$ with typicality is considered (called $\alcFt$) and the KLM properties of a preferential consequence relation are reformulated for this logic. Most of the postulates are satisfied, depending on the reformulation and on the chosen fuzzy combination functions.

The closure construction developed in \cite{JELIA2021} to define the models of a weighted defeasible knowledge base in the fuzzy case is reconsidered, 
by introducing a notion of {\em faithful} model of a weighted (fuzzy) knowledge base which is weaker than the notion of {\em coherent} fuzzy multipreference model in  \cite{JELIA2021}.  This allows us to capture the larger class of monotone non-decreasing activation functions in Multilayer Perceptrons (MLPs) \cite{Haykin99},
based on the idea that, in a deep neural network,  synaptic connections can be regarded as weighted conditionals.
The paper discusses the properties of faithful multipreference entailment for weighted conditionals in a fuzzy DL.

\section{The description logic $\alc$ and fuzzy $\alc$}  \label{sec:ALC}

In this section we recall the syntax and semantics of the description logic $\alc$ \cite{handbook} and of its fuzzy extension \cite{LukasiewiczStraccia09}. 

Let ${N_C}$ be a set of concept names, ${N_R}$ a set of role names
  and ${N_I}$ a set of individual names.  
The set  of $\alc$ \emph{concepts} (or, simply, concepts) can be
defined inductively: \\ 
- $A \in N_C$, $\top$ and $\bot$ are {concepts};\\
- if $C$ and $ D$ are concepts, and $r \in N_R$, then $C \sqcap D,\; C \sqcup D,\; \neg C, \; \forall r.C,\; \exists r.C$ 
are {concepts}.

\noindent
A knowledge base (KB) $K$ is a pair $({\cal T}, {\cal A})$, where ${\cal T}$ is a TBox and
${\cal A}$ is an ABox.
The TBox ${\cal T}$ is  a set of concept inclusions (or subsumptions) $C \sqsubseteq D$, where $C,D$ are concepts.
The  ABox ${\cal A}$ is  a set of assertions of the form $C(a)$ and $r(a,b)$
where $C$ is a  concept, $a$ and $b$ are individual names in $N_I$ and $r$ a role name in $N_R$.

An  $\alc$ {\em interpretation}  is defined as a pair $I=\langle \Delta, \cdot^I \rangle$ where:
$\Delta$ is a domain---a set whose elements are denoted by $x, y, z, \dots$---and 
$\cdot^I$ is an extension function that maps each
concept name $C\in N_C$ to a set $C^I \subseteq  \Delta$, 
each role name $r \in N_R$ to  a binary relation $r^I \subseteq  \Delta \times  \Delta$,
and each individual name $a\in N_I$ to an element $a^I \in  \Delta$.
It is extended to complex concepts  as follows:

$\top^I=\Delta$  \ \ \ \ \ \ \  

$\bot^I=\vuoto$  \ \ \ \ \ \ \  

$(\neg C)^I=\Delta \backslash C^I$

$(C \sqcap D)^I=C^I \cap D^I$

$(C \sqcup D)^I=C^I \cup D^I$

$(\exists r.C)^I =\{x \in \Delta \tc \exists y.(x,y) \in r^I \ \mbox{and}  \ y \in C^I\}$   \ \ \ \ \ \ \ 

$(\forall r.C)^I =\{x \in \Delta \tc \forall y. (x,y) \in r^I \Ri y \in C^I\}$   \ \ \ \ \ \ \ \ \

\noindent
The notion of satisfiability of a KB  in an interpretation and the notion of entailment are defined as follows:

\begin{definition}[Satisfiability and entailment] \label{satisfiability}
Given an $\alc$ interpretation $I=\langle \Delta, \cdot^I \rangle$: 

	- $I$  satisfies an inclusion $C \sqsubseteq D$ if   $C^I \subseteq D^I$;
	
	-   $I$ satisfies an assertion $C(a)$ (resp., $r(a,b)$) if $a^I \in C^I$ (resp.,  $(a^I,b^I) \in r^I$).

\noindent
 Given  a KB $K=({\cal T}, {\cal A})$, 
 an interpretation $I$  satisfies ${\cal T}$ (resp. ${\cal A}$) if $I$ satisfies all  inclusions in ${\cal T}$ (resp. all assertions in ${\cal A}$);
 $I$ is a \emph{model} of $K$ if $I$ satisfies ${\cal T}$ and ${\cal A}$.

 A subsumption $F= C \sqsubseteq D$ (resp., an assertion $C(a)$, $r(a,b)$),   {is entailed by $K$}, written $K \models F$, if for all models $I=$$\sx \Delta,  \cdot^I\dx$ of $K$,
$I$ satisfies $F$.

\end{definition}
Given a knowledge base $K$,
the {\em subsumption} problem is the problem of deciding whether an inclusion $C \sqsubseteq D$ is entailed by  $K$.

Fuzzy description logics have been widely studied in the literature for representing vagueness in DLs \cite{Straccia05,Stoilos05,LukasiewiczStraccia09,PenalosaARTINT15,BobilloOWL2EL2018}, 
based on the idea that concepts and roles can be interpreted 
as fuzzy sets. 
Formulas in Mathematical Fuzzy Logic \cite{Cintula2011} have a degree of truth in an interpretation rather than being true or false; similarly,
axioms in a fuzzy DL have a degree of truth, usually in the interval  $[0, 1]$. 
In the following we shortly recall the semantics of a fuzzy extension of $\alc$ referring to the survey by Lukasiewicz and Straccia \cite{LukasiewiczStraccia09}.
We limit our consideration  to a few features of a fuzzy DL  and, in particular, we omit considering datatypes.

A {\em fuzzy interpretation} for $\alc$ is a pair $I=\langle \Delta, \cdot^I \rangle$ where:
$\Delta$ is a non-empty domain and 
$\cdot^I$ is {\em fuzzy interpretation function} that assigns to each
concept name $A\in N_C$ a function  $A^I :  \Delta \ri [0,1]$,
to each role name $r \in N_R$  a function  $r^I:   \Delta \times  \Delta \ri [0,1]$,
and to each individual name $a\in N_I$ an element $a^I \in  \Delta$.
A domain element $x \in \Delta$ 
belongs to the extension of $A$ to some degree in $[0, 1]$, i.e., $A^I$ is a fuzzy set.

The  interpretation function $\cdot^I$ is extended to complex concepts as follows: 

$\mbox{\ \ \ }$ $\top^I(x)=1$, $\mbox{\ \ \ \ \ \ \ \ \ \  \ \ \ \ \ \ \ \ \ \ }$ $\bot^I(x)=0$,  $\mbox{\ \ \ \ \ \ \ \ \ \ \ \ \ \ \  \ \ \ \ \ }$  $(\neg C)^I(x)= \ominus C^I(x)$, 

$\mbox{\ \ \ }$  $(\exists r.C)^I(x) = sup_{y \in \Delta} \; r^I(x,y) \otimes C^I(y)$,  $\mbox{\ \ \ \ \ \ }$  $(C \sqcup D)^I(x) =C^I(x) \oplus D^I(x)$ 

$\mbox{\ \ \ }$  $(\forall r.C)^I (x) = inf_{y \in \Delta} \;  r^I(x,y) \rhd C^I(y)$, $\mbox{\ \ \ \ \ \ \ }$  $(C \sqcap D)^I(x) =C^I(x) \otimes D^I(x)$ 

\noindent
where  $x \in \Delta$ and $\otimes$, $\oplus$, $\rhd$ and $\ominus$ are arbitrary but fixed t-norm, s-norm, implication function, and negation function, chosen among the combination functions of various fuzzy logics 
(we refer to \cite{LukasiewiczStraccia09} for details). For instance, in both Zadeh and G\"odel logics $a \otimes b= min\{a,b\}$,  $a \oplus b= max\{a,b\}$. In Zadeh logic $a \rhd b= max\{1-a,b\}$ and $ \ominus a = 1-a$. In G\"odel logic  $a \rhd b= 1$ {\em if} $a \leq b$ {\em and} $b$ {\em otherwise};
 $ \ominus a = 1$ {\em if} $a=0$  {\em and} $0$ {\em otherwise}.

The  interpretation function $\cdot^I$ is also extended  to non-fuzzy axioms (i.e., to strict inclusions and assertions of an $\alc$ knowledge base) as follows:\\
 $(C \sqsubseteq D)^I= inf_{x \in \Delta}  C^I(x) \rhd D^I(x)$,
$\mbox{\ \ \ }$  $(C(a))^I=C^I(a^I)$,  $\mbox{\ \ \ }$  $(R(a,b))^I=R^I(a^I,b^I)$.

A {\em fuzzy $\alc$ knowledge base} $K$ is a pair $({\cal T}, {\cal A})$ where ${\cal T}$ is a fuzzy TBox  and ${\cal A}$ a fuzzy ABox. A fuzzy TBox is a set of {\em fuzzy concept inclusions} of the form $C \sqsubseteq D \;\theta\; n$, where $C \sqsubseteq D$ is an $\alc$ concept inclusion axiom, $\theta \in \{\geq,\leq,>,<\}$ and $n \in [0,1]$. A fuzzy ABox ${\cal A}$ is a set of {\em fuzzy assertions} of the form $C(a) \theta n$ or $r(a,b) \theta n$, where $C$ is an $\alc$ concept, $r\in N_R$, $a,b \in N_I$,  $\theta \in \{{\geq,}\leq,>,<\}$ and $n \in [0,1]$.
Following Bobillo and Straccia  \cite{BobilloOWL2EL2018}, we assume that fuzzy interpretations are {\em witnessed}, i.e., the sup and inf are attained at some point of the involved domain.
The notions of satisfiability of a KB  in a fuzzy interpretation and of entailment are defined in the natural way.
\begin{definition}[Satisfiability and entailment for fuzzy KBs] \label{satisfiability}
A  fuzzy interpretation $I$ satisfies a fuzzy $\alc$ axiom $E$ (denoted $I \models E$), as follows, 
 for $\theta \in \{\geq,\leq,>,<\}$:

- $I$ satisfies a fuzzy $\alc$ inclusion axiom $C \sqsubseteq D \;\theta\; n$ if $(C \sqsubseteq D)^I \theta\; n$;

- $I$ satisfies a fuzzy $\alc$ assertion $C(a) \; \theta \; n$ if $C^I(a^I) \theta\; n$;
 
- $I$ satisfies a fuzzy $\alc$ assertion $r(a,b) \; \theta \; n$ if $r^I(a^I,b^I) \theta\; n$.

\noindent
Given  a fuzzy KB $K=({\cal T}, {\cal A})$,
 a fuzzy interpretation $I$  satisfies ${\cal T}$ (resp. ${\cal A}$) if $I$ satisfies all fuzzy  inclusions in ${\cal T}$ (resp. all fuzzy assertions in ${\cal A}$).
A fuzzy interpretation $I$ is a \emph{model} of $K$ if $I$ satisfies ${\cal T}$ and ${\cal A}$.
A fuzzy axiom $E$   {is entailed by a fuzzy knowledge base $K$}, written $K \models E$, if for all models $I=$$\sx \Delta,  \cdot^I\dx$ of $K$,
$I$ satisfies $E$.
\end{definition}

\section{Fuzzy $\alc$ with typicality: $\alcFt$} \label{sec:fuzzyalc+T}

In this section, we extend fuzzy $\alc$ with typicality concepts of the form $\tip(C)$, for $C$ a concept in fuzzy $\alc$.
The idea is similar to the extension of $\alc$ with typicality \cite{FI09}, 
but transposed to the fuzzy case. The extension allows for the definition of {\em fuzzy typicality inclusions} of the form
$\tip(C) \sqsubseteq D \;\theta \;n$,  
meaning that typical $C$-elements are $D$-elements with a degree greater than $n$. A typicality  inclusion $\tip(C) \sqsubseteq D$, as in the two-valued case, stands for a KLM conditional implication $C \ent D$ \cite{KrausLehmannMagidor:90,whatdoes}, but now it has an associated degree.

We call $\alcFt$ the extension of fuzzy $\alc$ with typicality.
As in the two-valued case,  
such as in  $\sroiqpt$, a  preferential extension of ${\cal SROIQ}$ with typicality \cite{ISMIS2015},  or in the propositional typicality logic, PTL \cite{BoothCasiniAIJ19} 
the typicality concept may be allowed to freely occur within inclusions and assertions, while the nesting of the typicality operator is not allowed.

We have to define the semantics of $\alcFt$.
Observe that, in a fuzzy $\alc$ interpretation $I= \langle \Delta, \cdot^I \rangle$, the degree of membership $C^I(x)$ of the domain elements $x$ in a concept $C$, induces a preference relation $<_C$ on $\Delta$, as follows:
\begin{equation}\label{def:induced_order}
x <_C y \mbox{ iff } C^I(x) > C^I(y)
\end{equation}
Each $<_{C}$ has the properties of preference relations in KLM-style ranked interpretations \cite{whatdoes}, that is,  $<_{C}$ is a modular and well-founded strict partial order. 
Let us recall that, $<_{C}$ is {\em well-founded} 
if there is no infinite descending chain $x_1 <_C x_0$, $x_2 <_C x_1$, $x_3 <_C x_2, \ldots $ of domain elements;
    $<_{C}$ is {\em modular} if,
for all $x,y,z \in \Delta$, $x <_{C} y$ implies ($x <_{C} z$ or $z <_{C} y$).
Well-foundedness holds for the induced preference $<_C$ defined by condition (\ref{def:induced_order}) under the assumption that  fuzzy interpretations are witnessed \cite{BobilloOWL2EL2018} (see Section \ref{sec:ALC}) or that $\Delta$ is finite. 
In the following,  we will assume $\Delta$ to be finite.

While each preference relation $<_C$ has  the properties of a preference relation in KLM  rational interpretations \cite{whatdoes} (also called ranked interpretations), here there are
multiple preferences and, therefore, fuzzy interpretations can be regarded as {\em multipreferential} interpretations, which have also been studied in the two-valued case \cite{iclp2020,Delgrande2020,AIJ21}. 

Each preference relation $<_C$ captures the relative typicality of domain elements wrt concept $C$ and may then be used to identify the {\em typical  $C$-elements}. We will regard typical $C$-elements as the domain elements $x$ that  are preferred with respect to relation $<_C$
among those such that $C^I(x) \neq 0$.

Let $C^I_{>0}$ be the crisp set containing all domain elements $x$ such that $C^I(x)>0$, that is, $C^I_{>0}= \{x \in \Delta \mid C^I(x)>0 \}$.
One can provide a (two-valued) interpretation of typicality concepts $\tip(C)$ in a fuzzy interpretation $I$, by letting:
\begin{align}\label{eq:interpr_typicality}
	(\tip(C))^I(x)  & = \left\{\begin{array}{ll}
						 1 & \mbox{ \ \ \ \  if } x \in min_{<_C} (C^I_{>0}) \\
						0 &  \mbox{ \ \ \ \  otherwise } 
					\end{array}\right.
\end{align} 
where $min_<(S)= \{u: u \in S$ and $\nexists z \in S$ s.t. $z < u \}$.  When $(\tip(C))^I(x)=1$, we say that $x$ is a typical $C$-element in $I$.

Note that, if $C^I(x)>0$ for some $x \in \Delta$,  
$min_{<_C} (C^I_{>0})$ is non-empty.
This generalizes the property that, in the crisp case, $C^I\neq \emptyset$ implies  $(\tip(C))^I\neq \emptyset$.

\begin{definition}[$\alcFt$ interpretation]
An $\alcFt$ interpretation $I= \langle \Delta, \cdot^I \rangle$ is  fuzzy $\alc$ interpretation, equipped with the valuation of typicality concepts given by condition (\ref{eq:interpr_typicality}) above.
\end{definition}

The fuzzy interpretation  $I= \langle \Delta, \cdot^I \rangle$ implicitly defines a multipreference interpretation, where any concept $C$ is associated to a preference  relation $<_C$.  This is different from 
the two-valued multipreference semantics in \cite{iclp2020}, where only a subset of distinguished concepts have an associated preference, 
and a notion of global preference $<$ is introduced to define the interpretation of the typicality concept $\tip(C)$, for an arbitrary $C$. Here, we do not need to introduce a notion of global preference. The interpretation of any $\alc$ concept $C$ is defined compositionally from the interpretation of atomic concepts, and the preference relation $<_C$ associated to $C$ is defined from $C^I$.

The notions of {\em satisfiability} in $\alcFt$,   {\em model} of an $\alcFt$ knowledge base, and   $\alcFt$ {\em entailment} can be defined in a similar way as in fuzzy $\alc$ (see Section  \ref{sec:ALC}). 
In particular, given an $\alcFt$ knowledge base $K$, 
an inclusion $\tip(C) \sqsubseteq D \; \theta n$ (with $\theta \in \{\geq, \leq, >, <\}$ and $n \in [0,1]$) is {\em entailed 
from $K$} in $\alcFt$ (written $K \models_{\alcFt} \tip(C) \sqsubseteq D$) 
if $\tip(C) \sqsubseteq D \; \theta n$ is satisfied in all $\alcFt$ models  
$I$ of $K$.
For instance, the fuzzy inclusion axiom $\langle \tip(C) \sqsubseteq D \geq n \rangle$  is satisfied in a fuzzy interpretation $I= \langle \Delta, \cdot^I \rangle$ if  $inf_{x \in \Delta}  (\tip(C))^I(x) \rhd D^I(x) \geq n$ holds,  
which can be evaluated based on the combination functions of some specific fuzzy logic.

\section{KLM properties of $\alcFt$} 

In this section we aim at investigating the properties of typicality in $\alcFt$ and, in particular, verifying whether  the KLM postulates of a preferential consequence relation \cite{KrausLehmannMagidor:90,whatdoes} are satisfied in $\alcFt$.
The satisfiability of KLM postulates of rational or preferential consequence relations \cite{KrausLehmannMagidor:90,whatdoes} has been studied for $\alc$ with defeasible inclusions and typicality inclusions in the two-valued case \cite{sudafricaniKR,FI09}. The KLM postulates of a preferential consequence relation (i.e., reflexivity, left logical equivalence, right weakening, and, or, cautious monotonicity)
can be reformulated for $\alc$ with typicality, by considering that a typicality inclusion $\tip(C) \sqsubseteq D$ stands for a conditional $C {\ent} D$ in KLM preferential logics, as follows:

\begin{quote}

(REFL) \ $\tip(C) \sqsubseteq C $

(LLE) \ If $\models A \equiv B$ and $\tip(A) \sqsubseteq C $, then $\tip(B) \sqsubseteq C $ 

(RW) \  If $\models C \sqsubseteq D$ and $\tip(A) \sqsubseteq C $, then $\tip(A) \sqsubseteq D $ 

(AND) \ If $\tip(A) \sqsubseteq C $ and $\tip(A) \sqsubseteq D $, then $\tip(A) \sqsubseteq C \sqcap D $ 

(OR) \ If $\tip(A) \sqsubseteq C $ and $\tip(B) \sqsubseteq C $, then $\tip(A \sqcup B) \sqsubseteq C $

(CM) \  If $\tip(A) \sqsubseteq D$ and $\tip(A) \sqsubseteq C $, then $\tip(A \sqcap D) \sqsubseteq C $

\end{quote}
For $\alc$, $\models A \equiv B$ is interpreted as equivalence of concepts $A$ and $B$ in the underlying description logic $\alc$ (i.e., $A^I = B^I$ in all $\alc$ interpretations $I$), while $\models C \sqsubseteq D$
is interpreted as validity of the inclusion $C \sqsubseteq D$ in $\alc$ (i.e.,  $A^I \subseteq B^I$ for all $\alc$ interpretations $I$).

How can these postulates be reformulated in the fuzzy case? First, we can interpret $\models C \sqsubseteq D$ as the requirement that the fuzzy inclusion $C \sqsubseteq D \geq 1$ is valid in fuzzy $\alc$ (that is, $C \sqsubseteq D \geq 1$ is satisfied in all fuzzy $\alc$ interpretations),
and $\models A\equiv B$ as the requirement that the fuzzy inclusions $A \sqsubseteq B \geq 1$ and $B \sqsubseteq A \geq 1$ are valid in fuzzy $\alc$.
For the typicality inclusions, we have some options. We might interpret an inclusion $\tip(A) \sqsubseteq C $ as the fuzzy inclusion $\tip(A) \sqsubseteq C \geq 1$, or as the fuzzy inclusion $\tip(A) \sqsubseteq C > 0$.

The fuzzy inclusion axiom $\tip(A) \sqsubseteq C \geq 1$, is rather strong, as it requires that all typical $A$-elements belong to $C$ with membership degree $1$. On the other hand, $\tip(A) \sqsubseteq C > 0$ is a weak condition, as it requires that all typical $A$-elements belong to $C$ with some membership degree greater that $0$. We will see that both options fail to satisfy one of the postulates.
With the first option, the postulates can be reformulated as follows:

\begin{quote}
$(REFL')$ \ $\tip(C) \sqsubseteq C \geq 1$ 

$(LLE')$ \ If 
$\models A \equiv B$  and $\tip(A) \sqsubseteq C \geq 1$, then $\tip(B) \sqsubseteq C  \geq 1$ 

$(RW')$ \  If $\models C \sqsubseteq D$ and $\tip(A) \sqsubseteq C \geq 1$, then $\tip(A) \sqsubseteq D \geq 1$ 

$(AND')$ \ If $\tip(A) \sqsubseteq C \geq 1 $ and $\tip(A) \sqsubseteq D \geq 1$, then $\tip(A) \sqsubseteq C \sqcap D \geq 1$

$(OR')$ \ If $\tip(A) \sqsubseteq C \geq 1$ and $\tip(B) \sqsubseteq C \geq 1$, then $\tip(A \sqcup B) \sqsubseteq C \geq 1$

$(CM')$ \  If $\tip(A) \sqsubseteq D \geq 1$ and $\tip(A) \sqsubseteq C \geq 1$, then $\tip(A \sqcap D) \sqsubseteq C \geq 1$ 
\end{quote}
We can prove that, in the well-known Zadeh logic and G\"odel logic, the postulates above, with the exception of reflexivity $(REFL')$, are satisfied in all $\alcFt$ interpretations.

\begin{proposition} \label{prop:KLM_properties}
In Zadeh logic and in G\"odel logic any $\alcFt$ interpretation $I=\langle \Delta, \cdot^I \rangle $
satisfies postulates $(LLE'), (RW'),$ $ (AND'), (OR')$ and $(CM')$.

\end{proposition} 
\begin{proof}
Let  $I= \langle \Delta, \cdot^I \rangle$ be an $\alcFt$ interpretation 
in Zadeh logic, or in G\"odel logic.
Let us prove, as an example, that $I$ satisfies postulate $(LLE')$.  

\noindent
Assume that axioms $A \sqsubseteq B \geq 1$, $B \sqsubseteq A \geq 1$ are valid in fuzzy $\alc$ and that $\tip(A) \sqsubseteq C \geq 1 $ is satisfied in $I$.  
We prove that $\tip(B) \sqsubseteq C \geq 1$ is satisfied in $I$, that is $(\tip(B) \sqsubseteq C)^I \geq 1$.

From the validity of $A \sqsubseteq B \geq 1$ and $B \sqsubseteq A \geq 1$,  
$inf_{x \in \Delta} A^I(x) \rhd B^I(x) \geq 1$ and $inf_{x \in \Delta} B^I(x) \rhd A^I(x) \geq 1$ in both Zadeh logic and in G\"odel logic.
Hence,
\begin{equation}\label{eq:LLC}
\mbox{for all } x \in \Delta, \; A^I(x) \rhd B^I(x) \geq 1 \mbox{ and } B^I(x) \rhd A^I(x) \geq 1
\end{equation}
In G\"odel logic, this implies that: for all $x \in \Delta$, $A^I(x) \leq B^I(x)$ and $B^I(x) \leq A^I(x)$, i.e., $A^I(x) = B^I(x)$. 
Therefore, the preference relations $<_A$ and $<_B$ must be the same and $A^I_{>0}= B^I_{>0}$.
Hence, $\tip(A)^I(x) = \tip(B)^I(x)$ for all $x \in \Delta$, and from $(\tip(A) \sqsubseteq C)^I \geq 1 $, it follows that $(\tip(B) \sqsubseteq C)^I \geq 1 $, that is,  $\tip(B) \sqsubseteq C \geq 1$ is satisfied in $I$.

In Zadeh logic,  (\ref{eq:LLC}) implies that, for all $x \in \Delta$, $max\{1-A^I(x), B^I(x)\}\geq 1$ and $max\{1-B^I(x), A^I(x)\}\geq 1$ must hold, which implies that 
either  $A^I(x)=B^I(x)=0$ or $A^I(x)=B^I(x)=1$.
It follows that, for all $x \in \Delta$, either  $(\tip(A))^I(x)=(\tip(B))^I(x)=0$ or  $(\tip(A))^I(x)=(\tip(B))^I(x)=1$.
Hence, for all $x \in \Delta$,  it holds that $(\tip(A))^I(x)=(\tip(B))^I(x)$, and from $(\tip(A) \sqsubseteq C)^I \geq 1 $, it follows that $\tip(B) \sqsubseteq C \geq 1$ is satisfied in $I$. 

For the other postulates the proof is similar and  is omitted for lack of space.
\qed
\end{proof}
The meaning of the postulate  $(REFL')$ $\tip(C) \sqsubseteq C \geq 1$ is that the typical $C$-elements must have a degree of membership in $C$ equal to $1$, which may not be the case in an interpretation $I$, when there is no domain element $x$ such that $C^I(x)=1$.
Product and Lukasiewicz logics fail to satisfy postulates $(REFL')$ and $(OR')$, but they can be proven to satisfy  
all the other postulates.

The following corollary is a consequence of Proposition \ref{prop:KLM_properties}.
\begin{corollary}
In Zadeh logic and in G\"odel logic, $\alcFt$ entailment from a given knowledge base $K$ satisfies postulates $ (LLE'), (RW'),$ $ (AND'), (OR')$ and $(CM')$.
\end{corollary}
For instance, for $(AND')$, if $\tip(A) \sqsubseteq C \geq 1 $ and $\tip(A) \sqsubseteq D \geq 1$ are entailed from a knowledge base $K$ in $\alcFt$, then they are satisfied in all the models $I$ of $K$. Hence,  by Proposition  \ref{prop:KLM_properties}, $\tip(A) \sqsubseteq C \sqcap D \geq 1$ is as well satisfied in all the $\alcFt$ models of $K$, i.e., it is entailed by $K$.
As reflexivity is not satisfied, the notion of $\alcFt$ entailment from a given knowledge base $K$ does not define a preferential consequence relation, under the proposed formulation of the postulates.
It is easy to see that $\alcFt$ entailment  does not satisfy the Rational Monotonicity postulate.

Let us consider the alternative formulation of the postulates in the fuzzy case, obtained by interpreting the typicality inclusion $\tip(A) \sqsubseteq C $ as the fuzzy inclusion axiom $\tip(A) \sqsubseteq C > 0$, that is:

\begin{quote}
$(REFL'')$ \ $\tip(C) \sqsubseteq C >0$

$(LLE'')$ \ If 
$\models A \equiv B$  and $\tip(A) \sqsubseteq C > 0$, then $\tip(B) \sqsubseteq C  > 0$ 

$(RW'')$ \  If $\models C \sqsubseteq D$ and $\tip(A) \sqsubseteq C  > 0$, then $\tip(A) \sqsubseteq D  > 0$ 

$(AND'')$ \ If $\tip(A) \sqsubseteq C > 0 $ and $\tip(A) \sqsubseteq D > 0$, then $\tip(A) \sqsubseteq C \sqcap D > 0$

$(OR'')$ \ If $\tip(A) \sqsubseteq C > 0$ and $\tip(B) \sqsubseteq C > 0$, then $\tip(A \sqcup B) \sqsubseteq C > 0$

$(CM'')$ \  If $\tip(A) \sqsubseteq D > 0$ and $\tip(A) \sqsubseteq C > 0$, then $\tip(A \sqcap D) \sqsubseteq C > 0$ 
\end{quote}
With this formulation of the postulates, it can be proven that 
in Zadeh logic and G\"odel logic all postulates except for Cautious Monotonicity $(CM'')$ are satisfied in all $\alcFt$ interpretations. 
Under this formulation,  reflexivity $(REFL'')$ is satisfied in all  $\alcFt$ interpretations $I$,
as it requires that all typical $C$-elements in $I$ have a degree of membership in $C$ higher than $0$, which holds 
from the definition of $(\tip(C))^I$.

\begin{proposition}  \label{prop:KLM_properties2}
In Zadeh logic and in G\"odel logic, any $\alcFt$ interpretation $I$ satisfies postulates (REFL''), (LLE''), (RW''), (AND'') and (OR'').
\end{proposition}

\noindent
Cautious Monotonicity is too strong in the formulation $(CM'')$. From the hypothesis that $\tip(A) \sqsubseteq D > 0$ is entailed from $K$, we know that, in all $\alcFt$ models $I$ of $K$, the typical $A$-elements have some degree of membership $n$ in $D$. However, the degree $n$ may be small and not enough to conclude that typical $A \sqcap D$-elements are as well typical $A$-elements (which is needed to conclude 
that $\tip(A \sqcap D) \sqsubseteq C > 0$ is satisfied in $I$, given that  $\tip(A) \sqsubseteq C > 0$ is satisfied in $I$).
A weaker alternative formulation of Cautious Monotonicity 
can be obtained by strengthening the antecedent of $(CM'')$ as follows:
\begin{quote} 
$(CM^*)$ \  If $\tip(A) \sqsubseteq D \geq 1$ and $\tip(A) \sqsubseteq C >0$, then $\tip(A \sqcap D) \sqsubseteq C >0$ 
\end{quote}
This postulate is satisfied by  $\alcFt$ entailment in Zadeh logic and G\"odel logic. We can then prove that:

\begin{corollary}
In Zadeh logic and in G\"odel logic, $\alcFt$ entailment from a knowledge base $K$ satisfies postulates (REFL''), (LLE''), (RW''), (AND''), (OR'') and (CM$^*$).
\end{corollary}

As in the two-valued case,  the typicality operator $\tip$ introduced in $\alcFt$   
is non-monotonic in the following sense: for a given knowledge base $K$, from $K \models_{\alcFt} C \sqsubseteq D \geq 1$ we cannot conclude that
$K \models_{\alcFt} \tip(C) \sqsubseteq \tip(D) \geq 1$.
Nevertheless, the logic $\alcFt$ is monotonic, 
that is,  for two $\alcFt$ knowledge bases $K$ and $K'$, if $K \subseteq K'$, and $K \models_{\alcFt} E$ then 
$K' \models_{\alcFt} E$. 
 $\alcFt$ is a fuzzy relative of the monotonic logic $\alct$ \cite{FI09}.

Although most of the postulates of a preferential consequence relation hold in $\alcFt$,
this typicality extension of fuzzy $\alc$ is rather weak, as it happens in the two-valued case for
$\alct$, for the rational extension of $\alc$ in \cite{sudafricaniKR} and for preferential and rational entailment  in KLM approach \cite{whatdoes}. In particular, $\alcFt$ does not allow to deal with {\em irrelevance}. From the fact that birds normally fly, one would like to be able to conclude that normally yellow birds fly (being the color irrelevant to flying).  

As in KLM framework, in the two-valued case, this has led to the definition of 
non-monotonic defeasible DLs \cite{casinistraccia2010,CasiniDL2013,AIJ15,BonattiSauro17,Casinistraccia2012,AIJ21},  which exploit some closure construction (such as the rational closure   \cite{whatdoes} and the lexicographic closure \cite{Lehmann95}) or some notion of minimal entailment \cite{bonattilutz}. 
In the next section we strengthen $\alcFt$ based on a closure construction similar to the one in \cite{JELIA2021}, but exploiting a weaker notion of coherence, and we discuss its properties.

\section{Strengthening $\alcFt$: a closure construction} \label{sec:closure}

To overcome the weakness of rational closure (as well as of preferential entailment), Lehmann has introduced the lexicographic closure of a conditional knowledge base \cite{Lehmann95} which strengthens the rational closure by allowing further inferences.
From the semantic point of view, in the propositional case, a preference relation is defined on the set of propositional interpretations, so that the interpretations satisfying conditionals with higher rank are preferred to the interpretations satisfying conditionals with lower rank and, 
in case of contradictory defaults with the same rank,  interpretations satisfying more defaults with that rank are preferred.
The ranks of conditionals used by the  lexicographic closure construction are those computed by the rational closure construction \cite{whatdoes} and capture  specificity:  the higher is the rank, the more specific is the default.
In other cases, the ranks may be part of the knowledge base specification, such as for ranked knowledge bases in Brewka's framework of basic preference descriptions  \cite{Brewka04}, or might be learned from empirical data. 

In this section, we consider weighted (fuzzy) knowledge bases, where typicality inclusions are associated to weights, 
and develop a (semantic) closure construction to strengthen $\alcFt$ entailment,
which leads to a generalization of the notion of fuzzy coherent multipreference model in \cite{JELIA2021}. 
The construction is related to the definition of Kern-Isberner's c-representations  \cite{Kern-Isberner01,Kern-Isberner2014} which also include penalty points  
for falsified conditionals.

A  {\em weighted $\alcFt$ knowledge base} $K$, over a set ${\cal C}= \{C_1, \ldots, C_k\}$ of distinguished $\alc$ concepts,
is a tuple $\langle  {\cal T}_{f}, {\cal T}_{C_1}, \ldots, {\cal T}_{C_k}, {\cal A}_f  \rangle$, where  ${\cal T}_{f}$  is a set of fuzzy $\alcFt$ inclusion axiom, ${\cal A}_f$ is a set of fuzzy $\alcFt$ assertions  
and
${\cal T}_{C_i}=\{(d^i_h,w^i_h)\}$ is a set of all weighted typicality inclusions $d^i_h= \tip(C_i) \sqsubseteq D_{i,h}$ for $C_i$, indexed by $h$, where each inclusion $d^i_h$ has weight $w^i_h$, a real number.
As in \cite{JELIA2021}, the typicality operator is assumed to occur only on the left hand side of a weighted typicality inclusion, and we call {\em distinguished concepts}  those concepts $C_i$ occurring on the l.h.s. of some typicality inclusion $\tip(C_i) \sqsubseteq D$.
Arbitrary $\alcFt$ inclusions and assertions may belong to ${\cal T}_{f}$ and ${\cal A}_{f}$.

\begin{example} \label{exa:Penguin}
Consider the weighted knowledge base $K =\langle {\cal T}_{f},  {\cal T}_{Bird}, {\cal T}_{Penguin},$ $ {\cal T}_{Canary},$ $ {\cal A}_f \rangle$, over the set of distinguished concepts ${\cal C}=\{\mathit{Bird, Penguin, Canary}\}$, with empty ABox 
and with $ {\cal T}_{f}$ containing, for instance, the inclusions:

\noindent
\ \ $\mathit{Yellow \sqcap Black  \sqsubseteq  \bot} \geq 1$\footnote{This is a strong requirement which, e.g., in G\"odel logic, holds in $I$ only if $\mathit{(Yellow \sqcap Black)^I(x)=0 }$, $\forall x\in \Delta$. This suggests to be cautious when combining fuzzy inclusions and defeasible inclusions in the same KB (e.g.,  $\mathit{Penguin  \sqsubseteq  Bird \geq 1}$ would be too strong and conflicting with our interpretation of degree of membership as degree of typicality).}
 \ \ \ \ \ \ \ \ \ \ \   $\mathit{Yellow \sqcap Red  \sqsubseteq  \bot \geq 1}$   \ \ \ \ \ \ \ \ \ \  $\mathit{Black \sqcap Red  \sqsubseteq  \bot \geq 1}$

\noindent
The weighted TBox ${\cal T}_{Bird} $ 
contains the following weighted defeasible inclusions: 

$(d_1)$ $\mathit{\tip(Bird) \sqsubseteq Fly}$, \ \  +20  \ \ \ \ \ \ \ \ \  \ \ \ \ \   $(d_2)$ $\mathit{\tip(Bird) \sqsubseteq \exists has\_Wings. \top}$, \ \ +50

$(d_3)$ $\mathit{\tip(Bird) \sqsubseteq  \exists has\_Feather.\top}$, \ \ +50;

\noindent
${\cal T}_{Penguin}$ and  ${\cal T}_{Canary}$ contain, respectively, the following defeasible inclusions:

$(d_4)$ $\mathit{\tip(Penguin) \sqsubseteq Bird}$, \ \ +100 \ \ \ \ \ \ \ \ \ \ \ $(d_7)$ $\mathit{\tip(Canary) \sqsubseteq Bird}$, \ \ +100

$(d_5)$ $\mathit{\tip(Penguin) \sqsubseteq  Fly}$, \ \ - 70   \ \ \ \ \ \ \ \ \ \ \ \ \ \ \  $(d_8)$ $\mathit{\tip(Canary) \sqsubseteq Yellow}$, \ \  +30

$(d_6)$ $\mathit{\tip(Penguin) \sqsubseteq Black}$, \ \  +50; \ \ \ \ \ \ \ \ \ \ $(d_9)$ $\mathit{\tip(Canary) \sqsubseteq Red}$, \ \  +20

\noindent
 The meaning is that a bird normally has wings, has feathers and flies, but having wings and feather (both with weight 50)  for a bird is more plausible than flying (weight 20), although flying is regarded as being plausible. For a penguin, flying is not plausible (inclusion $(d_5)$ has negative weight -70), while being a bird and being black are plausible properties of prototypical penguins, and $(d_4)$ and $(d_6)$ have positive weights (100 and 50, respectively). Similar considerations can be done for concept $\mathit{Canary}$. Given Reddy who is red, has wings, has feather and flies (all with degree 1) and Opus who has wings and feather (with degree 1), is black with degree 0.8 and does not fly ($\mathit{Fly^I(opus) = 0}$), considering the weights of defeasible inclusions, we may expect Reddy to be more typical than Opus as a bird, but less typical than Opus as a penguin. 
 
 \end{example}

We define the semantics of a weighted knowledge base trough a {\em semantic closure construction}, similar in spirit to Lehmann's lexicographic closure \cite{Lehmann95}, but more related to c-representations and, additionally, based on multiple preferences.
The construction allows a subset of the $\alcFt$ interpretations to be selected, 
the interpretations whose induced preference relations $<_{C_i}$, for the distinguished concepts $C_i$,  faithfully represent the defeasible part of the knowledge base $K$.

Let ${\cal T}_{C_i}=\{(d^i_h,w^i_h)\}$ be the set of weighted typicality inclusions $d^i_h= \tip(C_i) \sqsubseteq D_{i,h}$ associated to the distinguished concept $C_i$, and let $I=\langle \Delta, \cdot^I \rangle$ be a fuzzy $\alcFt$ interpretation.
In the two-valued case, we would associate to each domain element $x \in \Delta$ and each distinguished concept $C_i$, a weight $W_i(x)$ of $x$ wrt $C_i$ in $I$, by summing the weights of the defeasible inclusions satisfied by $x$.
However, as $I$ is a fuzzy interpretation, we do not only  distinguish between the typicality inclusions satisfied or  
falsified  by $x$;
 we also need to consider, for all inclusions $\tip(C_i) \sqsubseteq D_{i,h} \in {\cal T}_{C_i}$,  
the degree of membership of $x$ in $D_{i,h}$. 
Furthermore, in comparing the weight of domain elements with respect to $<_{C_i}$, we give higher preference to the domain elements belonging to $C_i$ (with a degree
greater than $0$), with respect to those 
not belonging to $C_i$ (having membership degree $0$). 

For each domain element $x \in \Delta$ and distinguished concept $C_i$, {\em the weight $W_i(x)$ of $x$ wrt $C_i$} in the $\alcFt$ interpretation $I=\langle \Delta, \cdot^I \rangle$ is defined as follows:
 \begin{align}\label{weight_fuzzy}
	W_i(x)  & = \left\{\begin{array}{ll}
						 \sum_{h} w_h^i  \; D_{i,h}^I(x) & \mbox{ \ \ \ \  if } C_i^I(x)>0 \\
						- \infty &  \mbox{ \ \ \ \  otherwise }  
					\end{array}\right.
\end{align} 
where $-\infty$ is added at the bottom of all real values.

The value of $W_i(x) $ is $- \infty $ when $x$ is not a $C$-element (i.e., $C_i^I(x)=0$). 
Otherwise, $C_i^I(x) >0$ and the higher is the sum $W_i(x) $, the more typical is the element $x$ relative to concept $C_i$.
How much $x$ satisfies a typicality property  $\tip(C_i) \sqsubseteq D_{i,h}$ depends on the value of $D_{i,h}^I(x) \in [0,1]$, which is weighted by $ w_h^i $ in the sum. 
In the two-valued case, $D_{i,h}^I(x) \in \{0,1\}$, and 
$W_i(x)$ is the sum of the weights of the typicality inclusions for $C$ satisfied by $x$, if $x$ is a $C$-element,  and is $-\infty $, otherwise.

\begin{example} \label{exa:penguin2}
Let us consider again Example \ref{exa:Penguin}.
Let $I$ be an $\alcFt$ interpretation such that $\mathit{Fly^I(reddy)  = (\exists has\_Wings. \top)^I (reddy)= (\exists has\_Feather. \top)^I (reddy)=1}$ and  
\newline $\mathit{Red^I(reddy) =1 }$,
i.e., Reddy   flies, has wings and feather and is red (and $\mathit{Black^I(reddy)}$ $=0$). Suppose further that $\mathit{Fly^I(opus) = 0}$ and $\mathit{ (\exists has\_Wings. \top)^I (opus)= (\exists has\_}$ $ \mathit{Feather. \top)^I (opus)=1 }$ and $\mathit{ Black^I(opus) =0.8}$, i.e., Opus does not fly, has wings and feather, and is black with degree 0.8. Considering the weights of typicality inclusions for $\mathit{Bird}$,  $\mathit{W_{Bird}(reddy)= 20+}$ $\mathit{50+50=120}$ and $\mathit{W_{Bird}(opus)= 0+50+}$ $50=100$.
This suggests that reddy should be more typical as a bird than opus.
On the other hand, if we suppose $\mathit{Bird^I(reddy)=1}$ and $\mathit{Bird^I(opus)}$ $=0.8$, then $\mathit{W_{Penguin}}$ $\mathit{(reddy)}$ $ \mathit{= 100-70+0=30}$ and $\mathit{W_{Penguin}(opus)= 0.8 \times 100-} $ $\mathit{ 0+0.8 \times 50}$ $\mathit{=120}$. 
This suggests that reddy should be less typical as a penguin than opus.
 \end{example}
We have seen in Section \ref{sec:fuzzyalc+T} that each fuzzy interpretation $I$ induces a preference relation for each concept and, in particular, it induces a preference   $<_{C_i}$ for each distinguished concept $C_i$. 
We further require that, if $x  <_{C_i}  y$, 
then $x$ must be more typical than $y$ wrt $C_i$, that is, 
 the weight $W_i(x)$ of $x$ wrt $C_i$ should be higher than the weight $W_i(y)$ of $y$ wrt $C_i$ (and $x$ should satisfy more properties or more plausible properties of typical $C_i$-elements with respect to $y$). 
This leads to the following definition of a fuzzy multipreference model of a weighted a $\alcFt$  knowledge base.

\begin{definition}[Faithful (fuzzy) multipreference model of $K$]\label{fuzzy_fm-model} 
Let $K=\langle  {\cal T}_{f},$ $ {\cal T}_{C_1}, \ldots,$ $ {\cal T}_{C_k}, {\cal A}_f  \rangle$ be  a weighted $\alcFt$ knowledge base  over  ${\cal C }$. 
A {\em  faithful (fuzzy) multipreference model} (fm-model)  of $K$ is  a fuzzy $\alcFt$ interpretation $I=\langle \Delta, \cdot^I \rangle$  
s.t.: 
\begin{itemize}
\item
$I$  satisfies  the fuzzy inclusions in $ {\cal T}_{f}$ and the fuzzy assertions in ${\cal A}_f$;
\item 
for all $C_i\in {\cal C}$,  the preference {\em $<_{C_i}$   is faithful to $ {\cal T}_{C_i}$}, that is: 
\begin{align}\label{pref_rel_fuzzy}
x  <_{C_i}  y & \Ri W_i(x) > W_i(y)  
\end{align}
\end{itemize}
\end{definition}
Let us consider again Example \ref{exa:penguin2}.

\begin{example}
Referring to Example  \ref{exa:penguin2} above, 
where $\mathit{Bird^I(reddy)=1}$, $\mathit{Bird^I(opus)}$ $=0.8$,  let us further assume that $\mathit{Penguin^I(reddy)=0.2}$ and $\mathit{Penguin^I(opus)=0.8}$.
Clearly, $reddy <_{Bird} opus$ and  $opus <_{Penguin} reddy$. For the interpretation $I$ to be faithful, it is necessary that the conditions $\mathit{W_{Bird}(reddy) > W_{Bird}(opus)}$ and $\mathit{W_{Penguin}}$ $\mathit{(opus) >}$ $\mathit{ W_{Penguin}(reddy)}$ hold; which is true, as seen in Example  \ref{exa:penguin2}.
On the contrary, if it were $\mathit{Penguin^I(reddy)=0.9}$, the interpretation $I$ would not be faithful. 
\end{example}

Notice that, requiring that the converse of condition (\ref{pref_rel_fuzzy}) also holds, gives the equivalence
$x  <_{C_i}  y$ {\em iff} $W_i(x) > W_i(y) $,
a stronger condition which would make the notion of faithful multipreference model of $K$ above coincide with the notion of {\em coherent fuzzy multipreference model} of $K$ introduced in \cite{JELIA2021}. 
Here, we have considered the weaker notion of faithfulness and a larger class of fuzzy multipreference models of a weighted knowledge base, compared to the class of coherent models. 
This allows a larger class of monotone non-decreasing activation functions in neural network models to be captured 
(for space limitation, for this result, we refer to  \cite{arXiv_JELIA2020}, Sec. 7).

The notion of {\em faithful multipreference entailment (fm-entailment)} from a weighted $\alcFt$ knowledge base $K$ can be defined in the obvious way.
\begin{definition}[fm-entailment] \label{fm-entailment}
A fuzzy axiom $E$   {is fm-entailed from a fuzzy weighted knowledge base $K$} ($K \models_{fm} E$)  if, for all fm-models $I=\langle \Delta, \cdot^I \rangle$ of $K$, $I$ satisfies $E$.
\end{definition}
From Proposition \ref{prop:KLM_properties2}, the next corollary follows as a simple consequence.

\begin{corollary}
In Zadeh logic and in G\"odel logic, fm-entailment from a given knowledge base $K$ satisfies postulates (REFL''), (LLE''), (RW''), (AND''), (OR'') and (CM$^*$).
\end{corollary}
To conclude the paper let us 
infomally describe how fuzzy multipreference entailment deals with irrelevance and avoids inheritance blocking, properties which have been considered  as desiderata for preferential logics of defeasible reasoning \cite{Weydert03,Kern-Isberner2014}. 

For ``irrelevance", we have already considered an example: if typical birds fly, we would like to conclude that  of  typical yellow birds fly, as the property of being yellow is irrelevant with respect to flying. Observe, that in Example  \ref{exa:penguin2}, we can conclude that Reddy is more typical than Opus  as a bird ($\mathit{reddy <_{Bird} Opus}$), as Opus does not fly, while Reddy flies. The relative typicality of Reddy and Opus wrt $\mathit{Bird}$ does not depend on their color, and we would obtain the same relative preferences if reddy were yellow rather than red. A formal proof of the irrelevance property would require the domain $\Delta$ to be large enough to contain some typical bird which is yellow, requiring, as usual in the two-valued case \cite{AIJ15}, a restriction to some canonical models.

The  fuzzy multipreference entailment is not subject to the problem called by Pearl  the ``blockage of property inheritance" problem  \cite{Pearl90},
and by Benferhat et al. the ``drowning problem"  \cite{BenferhatIJCAI93}. This problem affects rational closure and system Z \cite{Pearl90}, as well as  rational closure refinements. Roughly speaking, the problem is that property inheritance from classes to subclasses is not guaranteed.
If a subclass is exceptional with respect to a superclass for a given property, 
it does not inherit from that superclass any other property. 
For instance, referring to the typicality inclusions in Example \ref{exa:penguin2}, in the rational closure, typical penguins would not inherit the property of typical birds of having wings, being exceptional to birds concerning flying.
On the contrary,
in  fuzzy multipreference models, 
considering again Example \ref{exa:penguin2}, the degree of membership of a domain element $x$ in concept $\mathit{Bird}$, i.e., $\mathit{Bird^I(x)}$, is used to determine the weight of $x$ wrt $\mathit{Penguin}$ (as the weight of typicality inclusion $(d_4)$ is positive. The higher is the value of $\mathit{Bird^I(x)}$, the higher the value of  $\mathit{W_{Penguin}(x)}$.  Hence, provided the relevant properties of penguins (such as non-flying) remain unaltered, the more typical is $x$ as a bird, 
the more typical is $x$ as a Penguin.
Notice also that the weight $\mathit{W_{Bird}(x)}$ of a domain element $x$ wrt $\mathit{Bird}$ is related to the interpretation of $Bird$ in $I$ by the faithfulness condition.

\section{Conclusions}

In this paper we have studied the properties of an extension of fuzzy $\alc$ with typicality, $\alcFt$. We have considered some alternative reformulation of the KLM postulates of a preferential consequence relation for $\alcFt$, showing that most of these postulates are satisfied, depending on the formulation considered and on the fuzzy logic combination functions.
We have considered a (semantic) closure construction to strengthen $\alcFt$,  by defining a notion of faithful (fuzzy) multipreference model of a weighted knowledge base. Faithful models of a conditional (weighted) knowledge base are a more general class 
of models with respect to the coherent fuzzy multipreference models considered in \cite{JELIA2021} to provide a semantic interpretation of multilayer perceptrons. 
This allows us to capture the larger class of monotone non-decreasing activation functions in multilayer perceptrons (a result which is not included in the paper due to space limitations and for which we refer to \cite{arXiv_JELIA2020}). The paper studies the KLM properties of faithful multipreference entailment 
and discusses how the multipreference approach allows to deal with irrelevance and avoids inheritance blocking.

For MLPs, the proposed semantics allows the input-output behavior of a deep network (considered after training) to be captured by a fuzzy multipreference interpretation built over a set of input stimuli, through a simple construction which exploits the activity level of neurons for the stimuli. 
Each unit $h$ of $\enne$ can be associated to a concept name $C_h$ and,
for a given domain $\Delta$ of input stimuli, the activation value of unit $h$ for a stimulus $x$ is interpreted as the degree of membership of $x$ in concept $C_h$. 
The resulting fm-interpretation can be used for verifying properties of the network by model checking and it can be proven \cite{JELIA2021} to be a model of the conditional knowledge base $K^{\enne}$ obtained, from the network $\enne$,  by mapping synaptic connections to weighted conditionals.
This opens to the possibility 
of combining empirical knowledge and symbolic knowledge in the form of DL axioms  and motivates the study of the properties of this  
multipreference extension of fuzzy DLs.

Undecidability results for fuzzy description logics with general inclusion axioms 
\cite{BaaderPenaloza11,CeramiStraccia2011,BorgwardtPenaloza12} 
have motivated restricting the logics to finitely valued semantics \cite{BorgwardtPenaloza13}, 
and also motivate the investigation of decidable approximations of fm-entailment.   
An issue is whether alternative (non-crisp) definitions of the typicality operator could be adopted, inspired to fuzzy set based models of linguistic hedges \cite{HuynhHN02}. 
Another issue is whether the multipreference semantics can provide a semantic interpretation to other neural network models, besides
MLPs and 
Self-Organising Maps  \cite{kohonen2001}, for which a (two-valued) multipreference semantics and a fuzzy semantics have been 
investigated in  \cite{CILC2020}.


\end{document}